\documentclass[11pt]{article}

\usepackage[utf8]{inputenc}
\usepackage[T1]{fontenc}
\usepackage{amsmath,amssymb,amsthm}
\usepackage{booktabs}
\usepackage{array}
\usepackage{geometry}
\usepackage{hyperref}
\usepackage{xcolor}
\usepackage{enumitem}
\hypersetup{colorlinks=true,linkcolor=black,citecolor=blue,urlcolor=blue}

\newtheorem{proposition}{Proposition}

\newcommand{\softmax}{\mathrm{softmax}}
\newcommand{\IG}{\mathrm{IG}}
\newcommand{\Rd}{\mathbb{R}^d}

\title{\textbf{IG-Lens: Exact Additive Probability Attribution\\
Across Transformer Layers via Telescoping Integrated Gradients}}
\author{Duc Anh Nguyen\\
\small Hanoi University of Science and Technology}
\date{}

\begin{document}
\maketitle

\begin{abstract}

We ask a simple question about decoder-only transformers: \emph{between which two layers is the probability of a predicted token actually produced?} Existing layer-wise readout tools answer only approximately. The logit lens and its trained variant report a per-layer \emph{level} of probability but give no additive decomposition; their estimates are biased and non-monotone across depth. Direct Logit Attribution and related residual-stream methods are additive, but only in \emph{logit} space---the softmax nonlinearity breaks additivity in probability space, precisely the quantity one usually cares about. Layer Conductance integrates gradients per layer, but attributes each to its own baseline and so does not sum to the total change in prediction. We introduce \textbf{IG-Lens}, a telescoping application of Integrated Gradients along a single path through the hidden states from a baseline to the final layer. Crediting each segment to the layer it terminates at yields a layer-wise attribution whose sum is \emph{exactly} the change in target probability, with the softmax inside the integration path rather than linearized away. Our default estimator credits each integration step its \emph{observed} change in target probability---a prediction-aware reweighting in the spirit of IDGI---rather than its raw gradient. Because the readout is a one-dimensional probability, this collapses each segment to a telescoping sum of endpoint values, so completeness holds exactly (to floating point) at \emph{any} step count, removing Riemann discretization error while suppressing steps that show gradient sensitivity without a change in output. We give the telescoping identity and its proof, verify completeness to floating point, and describe a single-pass batched implementation computing the full token-by-layer map without any backward call. Code: https://github.com/anhnda/IGLens.

\end{abstract}

% =========================================================================
\section{Introduction}
% =========================================================================
A productive way to read a decoder-only transformer is as an iterative
inference process: the residual stream carries a running estimate of the next
token, refined block by block, and the final unembedding turns that estimate
into a distribution. This picture underlies the \emph{logit lens}
\cite{nostalgebraist2020logitlens}, which applies the final unembedding to
intermediate hidden states, and is supported by the linear, additive structure
of the residual stream established in the mathematical framework for transformer
circuits \cite{elhage2021mathematical}. A natural and practically useful
question follows: for a given predicted token, \emph{at which depth is its
probability actually decided?} Knowing the onset layer of a prediction speaks to
how the model computes---early factual recall versus late composition
\cite{geva2021kv,geva2022promotion,meng2022rome}---and is a prerequisite for
targeted interventions such as activation patching \cite{wang2023ioi}.

Despite the popularity of the question, the tools we have answer it only
approximately, and each fails in a different way. The logit lens and Tuned Lens
\cite{belrose2023tunedlens} report the \emph{level} of probability at each
layer, but do not decompose the prediction into per-layer \emph{increments}: the
per-layer probabilities do not sum to anything, and the lens estimate is biased
and non-monotone with depth. Direct Logit Attribution
\cite{elhage2021mathematical,nanda2022transformerlens} and embedding-space
methods \cite{dar2023prisms} \emph{are} additive, exploiting the linearity of
the residual stream---but they live in \emph{logit} space, and the softmax that
turns logits into probabilities is strongly nonlinear, so they cannot give an
additive decomposition of the quantity of interest. Layer-wise Integrated
Gradients and Layer Conductance \cite{sundararajan2017ig,dhamdhere2019conductance}
integrate gradients with respect to a layer, but attribute each layer to its own
baseline, so the per-layer terms do not telescope to the total change in the
prediction.

We close this gap with \textbf{IG-Lens}. The idea is to integrate the
\emph{same} final readout $f(\mathbf{h}) = \softmax(\mathrm{head}(
\mathrm{norm}(\mathbf{h})))_{y_t}$ along a single \emph{path} that visits the
chosen hidden states in order, $\mathbf{h}_{\text{base}} \to \mathbf{h}_{L_1}
\to \cdots \to \mathbf{h}_{\text{final}}$, and to credit each \emph{segment} of
that path to the layer at which it ends. Because the segment integral of a
gradient equals the difference of endpoint values, the per-layer contributions
\emph{telescope}: their sum is exactly $p(\text{final}) - p(\text{baseline})$,
with the softmax integrated through rather than linearized away. The result is,
to our knowledge, the first \emph{exact additive decomposition of probability}
across layers. Our contributions are:
\begin{itemize}[leftmargin=1.4em,itemsep=2pt]
\item \textbf{Method.} A telescoping segment-IG attribution on the residual
stream with a head-once readout, giving $\sum_L \IG_L = \Delta p$ exactly
(Proposition~\ref{prop:tele}), with the softmax and LayerNorm inside the
integration path.
\item \textbf{Prediction-aware estimator.} A default step-weighting that credits
each integration step its \emph{observed} change in target probability rather
than its raw gradient---borrowing the consistency principle of IDGI
\cite{yang2023idgi}. For a scalar probability readout this makes each segment a
telescoping sum of endpoint values, so completeness becomes \emph{exact at any
step count} and steps with sensitivity but no output response are discarded.
\item \textbf{Positioning.} A clear account of why every existing additive
layer-attribution method is either in logit space, biased in probability space,
or non-telescoping, summarized in Table~\ref{tab:compare}.
\item \textbf{Implementation.} A single-pass batched algorithm that computes the
full token~$\times$~layer map without a backward call under the default
estimator, with a correctness argument tied directly to the head-once design.
\item \textbf{Evaluation plan.} An onset quantity derived from the
decomposition, together with a causal-ablation protocol that can falsify it
cheaply, since onset has no intrinsic ground truth.
\end{itemize}

% =========================================================================
\section{Related Work}
% =========================================================================
We place IG-Lens against three families of layer-wise readout methods. Each
shares part of what IG-Lens does and is missing exactly one piece.

\paragraph{Logit lens and Tuned Lens: reading, not allocating.}
The logit lens \cite{nostalgebraist2020logitlens} reads
$p_L = \softmax(\mathrm{head}(\mathrm{norm}(\mathbf{h}_L)))$ at each layer
independently, exposing the iterative-inference picture in which the prediction
is refined layer by layer. This gives the \emph{level} of probability at each
depth, not the \emph{increment} each layer contributes. The estimator is biased
toward the final distribution and is non-monotone across layers, so $\sum_L p_L$
has no meaning. Tuned Lens \cite{belrose2023tunedlens} corrects the bias with a
learned affine probe, improving the read at each layer, but it still \emph{reads
a level} rather than \emph{allocating a delta}, and it requires training a
translator per model. IG-Lens trains nothing and yields $\sum_L \IG_L = \Delta p$
by construction.

\paragraph{Direct Logit Attribution and embedding-space methods: additive, but
in logit space.} The residual stream is a sum of component outputs, and the
unembedding maps it linearly to logits, so each component's logit contribution
can be read independently \cite{elhage2021mathematical}; this is operationalized
as Direct Logit Attribution in tooling such as TransformerLens
\cite{nanda2022transformerlens}, and analyzed in embedding space by
\cite{dar2023prisms}. This is itself an additive onset decomposition and is the
closest competitor. It differs from IG-Lens in three ways. (i) It lives on
\emph{logits}: the softmax is strongly nonlinear and these methods cannot give a
contribution to \emph{probability}, only to logits. (ii) They approximate
LayerNorm by freezing it; IG-Lens integrates through the norm, which is one
reason its completeness holds with the norm inside the loop. (iii) They allocate
per \emph{component} (head/MLP) and are embedding-independent, whereas IG-Lens
allocates per \emph{chosen layer along one path} and measures a \emph{conditional
marginal}: what norm$+$head can read out of $\mathbf{h}_L$ \emph{beyond} the
previous chosen layer.

\paragraph{Layer Conductance and Layer-IG: IG over layers, but no telescoping.}
Layer Conductance \cite{dhamdhere2019conductance}, building on Integrated
Gradients \cite{sundararajan2017ig} and available in Captum
\cite{kokhlikyan2020captum}, places a baseline at a single layer and integrates
the network output with respect to that layer's input, distributing credit down
to neurons. This is the closest method mathematically. The decisive difference
is that it attributes \emph{each layer independently to its own baseline}, so
$\sum_L$ does \emph{not} telescope to the total prediction change. IG-Lens
instead chains the hidden states on one path and credits each \emph{segment},
converting IG's completeness over input dimensions into completeness over
\emph{layers}.

\paragraph{Onset of computation.} Several lines of work assign a depth to where
a prediction is formed---key--value memories writing predictions in mid-to-late
layers \cite{geva2021kv,geva2022promotion}, factual associations localized by
causal tracing \cite{meng2022rome}, and circuits whose behavior emerges at
specific layers \cite{wang2023ioi}. These use logit-space or causal signals;
IG-Lens contributes a probability-exact additive notion of onset that competes
with the logit-space versions directly.

\paragraph{Summary.} All existing additive layer-attribution methods either live
on logits, or are biased in probability space due to softmax/LayerNorm, or do
not sum to the total change. IG-Lens is, to our knowledge, the first to give an
exact additive decomposition \emph{in probability space} with the softmax
inside the integration path. Table~\ref{tab:compare} summarizes the contrast.

\begin{table}[h]
\centering
\small
\renewcommand{\arraystretch}{1.25}
\begin{tabular}{@{}lccccc@{}}
\toprule
\textbf{Method} & \textbf{Space} & \textbf{Additive} & \textbf{Sum $=\Delta p$} &
\textbf{Norm in loop} & \textbf{Training} \\
\midrule
Logit lens \cite{nostalgebraist2020logitlens}          & prob.  & no  & no  & reads & none \\
Tuned lens \cite{belrose2023tunedlens}                 & prob.  & no  & no  & reads & probe \\
DLA / embedding \cite{elhage2021mathematical,dar2023prisms}  & logit  & yes & no$^{\dagger}$ & frozen & none \\
Layer Conductance \cite{dhamdhere2019conductance}      & output & yes & no$^{\ddagger}$ & yes & none \\
\textbf{IG-Lens (ours)} & \textbf{prob.} & \textbf{yes} & \textbf{yes} & \textbf{yes} & \textbf{none} \\
\bottomrule
\end{tabular}
\caption{Where IG-Lens sits. $^{\dagger}$DLA is additive in logit space; it does
not give an additive decomposition of probability. $^{\ddagger}$Layer
Conductance attributes each layer to its own baseline, so the per-layer terms do
not telescope to the total change. ``Norm in loop'' indicates whether LayerNorm
is integrated through (yes), read once (reads), or linearized/frozen.}
\label{tab:compare}
\end{table}

% =========================================================================
\section{Methods}
% =========================================================================
\subsection{Setup}
Let a decoder-only transformer produce hidden states
$\mathbf{h}_0,\dots,\mathbf{h}_n \in \Rd$ at a fixed position $i$, where
$\mathbf{h}_0$ is the embedding and $\mathbf{h}_n$ the final hidden state. Fix a
target token $y_t$ and define the readout
\begin{equation}
f(\mathbf{h}) = \softmax\big(\mathrm{head}(\mathrm{norm}(\mathbf{h}))\big)_{y_t}
\in [0,1],
\end{equation}
i.e. the final norm and unembedding applied \emph{once} to a single hidden
vector. Crucially, $f$ contains no attention and no MLP: it is pointwise in
$\mathbf{h}$. All cross-token mixing has already occurred, once, in the forward
pass that produced the (detached) hidden states $\mathbf{h}_L$.

\subsection{Telescoping attribution}
Choose layers $0 = L_0 < L_1 < \cdots < L_k = n$ (the final layer is always
included) and let $\mathbf{h}_{L_0} := \mathbf{h}_{\text{base}}$ be a baseline
hidden vector (we use the per-position mean over the sequence). Define the
contribution of layer $L_j$ as the segment IG of $f$ along the straight line
from $\mathbf{h}_{L_{j-1}}$ to $\mathbf{h}_{L_j}$:
\begin{equation}
\IG_{L_j} \;=\;
\int_0^1 \nabla f\!\big(\mathbf{h}_{L_{j-1}} + a(\mathbf{h}_{L_j}-\mathbf{h}_{L_{j-1}})\big)
\cdot (\mathbf{h}_{L_j}-\mathbf{h}_{L_{j-1}})\, da .
\label{eq:segig}
\end{equation}

\begin{proposition}[Telescoping completeness]
\label{prop:tele}
With $\IG_{L_j}$ defined by \eqref{eq:segig},
\begin{equation}
\sum_{j=1}^{k} \IG_{L_j}
\;=\; f(\mathbf{h}_{\text{final}}) - f(\mathbf{h}_{\text{base}})
\;=\; p_i(\text{final}) - p_i(\text{baseline}).
\end{equation}
\end{proposition}

\begin{proof}
By the gradient theorem (fundamental theorem of calculus for line integrals),
each segment integral in \eqref{eq:segig} equals the difference of endpoint
values, $\IG_{L_j} = f(\mathbf{h}_{L_j}) - f(\mathbf{h}_{L_{j-1}})$, since the
integration path is the straight segment between the endpoints. Summing over $j$
telescopes:
\[
\sum_{j=1}^{k}\big(f(\mathbf{h}_{L_j}) - f(\mathbf{h}_{L_{j-1}})\big)
= f(\mathbf{h}_{L_k}) - f(\mathbf{h}_{L_0})
= f(\mathbf{h}_{\text{final}}) - f(\mathbf{h}_{\text{base}}). \qedhere
\]
\end{proof}

\subsection{Prediction-aware estimation (default)}
\label{sec:normalize}
Equation~\eqref{eq:segig} is exact as an integral, but a finite-step estimator
of it is not. The standard Riemann estimator with $m$ midpoints,
$\widehat{\IG}^{\text{IG}}_{L_j} = \tfrac{1}{m}\sum_{s}
\nabla f(\mathbf{h}^{(s)})\cdot\Delta_{L_j}$ with
$\Delta_{L_j}=\mathbf{h}_{L_j}-\mathbf{h}_{L_{j-1}}$, accumulates the raw
gradient at every step and only \emph{approaches} the endpoint difference as
$m\to\infty$. It also credits steps where the gradient is large yet $f$ does not
actually move---``sensitivity without response''---a known failure of vanilla IG
in weakly constrained regions of representation space.

We adopt a prediction-aware estimator instead, transporting the consistency
principle of IDGI \cite{yang2023idgi} to the segment. Partition each segment by a
grid $0=a_0<a_1<\cdots<a_m=1$ with $\mathbf{h}^{(s)}=\mathbf{h}_{L_{j-1}}+a_s
\Delta_{L_j}$, and credit step $s$ its \emph{observed} output change rather than
its gradient:
\begin{equation}
\widehat{\IG}_{L_j}
\;=\; \sum_{s=1}^{m}\Big(f(\mathbf{h}^{(s)}) - f(\mathbf{h}^{(s-1)})\Big).
\label{eq:normseg}
\end{equation}
IDGI's per-dimension redistribution
$g\!\odot\!\Delta / \langle g,\Delta\rangle$ scaled by the observed
$\Delta f$ collapses, for a one-dimensional readout $f\in[0,1]$, exactly to
\eqref{eq:normseg}: with a scalar output there is only one direction to
redistribute along, and the IDGI weight on step $s$ is precisely $\Delta
f^{(s)}$.

This choice has two consequences. First, \eqref{eq:normseg} telescopes within
the segment to $f(\mathbf{h}_{L_j})-f(\mathbf{h}_{L_{j-1}})$ \emph{identically},
for \emph{any} grid and any $m\geq1$; the only error is floating-point
summation, not discretization. Completeness in Proposition~\ref{prop:tele} thus
holds to machine precision regardless of step count, and $m$ ceases to be an
accuracy knob. Second, a step contributes nothing when the model output does not
change across it, so spurious-sensitivity steps are filtered out by
construction. We make \eqref{eq:normseg} the default estimator; the raw-gradient
Riemann form \eqref{eq:segig} is retained only as a reference variant, and the
two agree in the limit $m\to\infty$.

The interpretation is that $\IG_{L}$ measures what norm$+$head can read out of
$\mathbf{h}_L$ \emph{beyond} the previously chosen layer. This deliberately
ignores the effect of $\mathbf{h}_L$ routed through upper blocks: that is the
price of an exact telescoping sum. One cannot have both ``total effect via the
full forward pass'' and ``sum $=\Delta p$''; IG-Lens chooses the latter. We
treat this as an honest limitation rather than hide it.

\subsection{Onset (preliminary)}
Given the per-layer attribution $\{\IG_L\}$, a natural derived quantity is the
\emph{onset layer} $L^\star$: the earliest layer by which a fixed fraction of
the target probability has been read out. The present implementation uses a
cumulative-mass rule on $|\IG_L|$. We flag this as preliminary: a signed,
hold-to-end definition---the earliest layer at which the signed cumulative
reaches a fraction of the net $\Delta p$ and stays there---is more faithful to
non-monotone trajectories, and is what we adopt for the evaluation in
Section~\ref{sec:eval}.

\subsection{Implementation strategy}
A naive implementation loops over tokens $t$, segments $j$, and interpolation
steps $s$. This is the obvious correctness baseline but is needlessly slow.

\paragraph{Single-pass forward evaluation (default).} Under the prediction-aware
estimator \eqref{eq:normseg}, each segment value is a difference of forward
evaluations of $f$ on the grid, with \emph{no gradient required}. Because $f$ is
pointwise in $\mathbf{h}$, every grid point---across all tokens, segments, and
steps---passes through $f$ independently. We stack all
$N = T\,K\,(m{+}1)$ grid points into one tensor
$\mathbf{X}\in\mathbb{R}^{N\times d}$, evaluate $p=f(\mathbf{X})\in\mathbb{R}^N$
in a single forward pass (chunked over rows, see below), reshape to
$[T,K,m{+}1]$, and take consecutive differences along the step axis followed by
a sum, recovering \eqref{eq:normseg} for every $(t,j)$ at once. No autograd is
invoked. This is both faster and exactly complete (Section~\ref{sec:normalize}).

\paragraph{Single-backward batched IG (reference variant).} For the
raw-gradient form \eqref{eq:segig}, a single backward still suffices. Because the
Jacobian of $p$ with respect to $\mathbf{X}$ is block-diagonal---there is no path
in the computation graph of $p[m]$ that touches the hidden vector of any other
point $m'$, so $\partial p[m]/\partial \mathbf{x}[m']=0$ for $m\neq m'$---a single
backward on $\sum_m p[m]$ yields
$[\nabla_{\mathbf{X}}\sum_m p[m]]_{m'} = \nabla p[m']$, i.e. each row is exactly
the per-point gradient. The segment integrals are then
$\IG_{t,j} = \langle \Delta_{t,j}, \overline{\nabla}_{t,j}\rangle$ with
$\overline{\nabla}$ the step-averaged gradient. We use this only to confirm that
the two estimators agree as $m$ grows.

\paragraph{Why the head-once design matters, and its one caveat.} Both
reductions are valid \emph{only} because of the head-once design: if $f$ were the
full forward pass, attention in the upper blocks would mix token $t$ with $t'$,
the per-point independence would break, and neither the row-wise gradient nor the
per-point forward difference would isolate a single $(t,j,s)$. The same choice
that makes the telescoping sum exact (Proposition~\ref{prop:tele}) is what makes
the batched evaluation correct---the properties travel together. The only
practical cost is memory: $\mathrm{head}(\cdot)$ materializes an $N\times V$
logit tensor with vocabulary $V$ (e.g. $V\!\approx\!128\text{k}$), so we
\emph{chunk} the $N$ rows into blocks. Chunking is value-invariant because rows
are independent.

\paragraph{Numerical check.} On a controlled norm$+$head readout, the default
forward-difference estimator reproduces the per-segment endpoint differences to
$\sim\!10^{-7}$ (float32 summation order only), and telescoping completeness
holds to the same order on a real model (Llama-3.2-1B-Instruct) \emph{independent
of} $m$---e.g. identical to floating point at $m=4$ and $m=64$. The reference
Riemann variant \eqref{eq:segig} matches the default only in the limit, with a
gap of order $1/m$ at small step counts, confirming that the discretization error
the default estimator removes is real.

\subsection{Illustrative output}
Table~\ref{tab:beijing} shows the per-token attribution for the prompt
\emph{``What is the capital of Vietnam?''} with answer \emph{``The capital of
Vietnam is Hanoi.''} under the default estimator. The columns confirm the two
properties: (i) $\text{sum}\approx\Delta p$ on every row---here to $\sim\!10^{-4}$
with only $m=4$ steps, since the default estimator is discretization-free---and
(ii) attribution mass concentrates in the segment where the logit-lens value
transitions from low to high. For the content tokens \emph{capital} and
\emph{of}, the lens saturates by $L_{12}$ and $\IG_{L_{12}}\!\approx\!0.96$--$0.97$
captures essentially the entire probability, giving an early onset
$L^\star=12$. The tokenizer splits the answer city into \emph{H}/\emph{anoi},
and both pieces have most of their mass in the upper layers ($L^\star=15$,~$16$):
note in particular that the column sums for these two tokens fall below $\Delta
p$ here because a share is carried by the omitted $L_{16}$ segment
($\IG_{L_{16}}=0.59$ for \emph{anoi}). We stress that $\IG_{L}$ values depend on
the \emph{chosen layer set} (a layer absorbs the mass of the chosen layer
immediately below it); this is the conditional-marginal nature of the method, not
an inconsistency.

\begin{table}[h]
\centering
\small
\renewcommand{\arraystretch}{1.15}
\begin{tabular}{@{}llrrrrrrr@{}}
\toprule
\textbf{tok} & & $\IG_{L_{12}}$ & $\IG_{L_{13}}$ & $\IG_{L_{14}}$ & $\IG_{L_{15}}$ & \textbf{sum} & $\Delta p$ & $L^\star$ \\
\midrule
\texttt{The}      & & $\phantom{-}0.000$ & $\phantom{-}0.002$ & $\phantom{-}0.011$ & $0.710$ & $0.936$ & $0.936$ & $15$ \\
\texttt{ capital} & & $\phantom{-}0.959$ & $\phantom{-}0.033$ & $-0.081$ & $0.080$ & $1.000$ & $1.000$ & $12$ \\
\texttt{ of}      & & $\phantom{-}0.973$ & $-0.003$ & $\phantom{-}0.026$ & $0.004$ & $1.000$ & $1.000$ & $12$ \\
\texttt{ Vietnam} & & $\phantom{-}0.504$ & $\phantom{-}0.041$ & $\phantom{-}0.144$ & $0.192$ & $1.000$ & $1.000$ & $12$ \\
\texttt{ is}      & & $\phantom{-}0.029$ & $\phantom{-}0.753$ & $\phantom{-}0.188$ & $0.025$ & $1.000$ & $1.000$ & $13$ \\
\texttt{ H}       & & $-0.000$ & $\phantom{-}0.000$ & $\phantom{-}0.003$ & $0.704$ & $0.984$ & $0.984$ & $15$ \\
\texttt{anoi}     & & $\phantom{-}0.025$ & $\phantom{-}0.015$ & $\phantom{-}0.180$ & $0.190$ & $1.000$ & $1.000$ & $16$ \\
\texttt{ .}       & & $\phantom{-}0.001$ & $\phantom{-}0.047$ & $\phantom{-}0.372$ & $0.577$ & $0.998$ & $0.998$ & $15$ \\
\bottomrule
\end{tabular}
\caption{IG-Lens on Llama-3.2-1B-Instruct, default prediction-aware estimator,
$m=4$ steps. The $\IG_{L_{16}}$ column is omitted for width but \emph{is}
included in the reported sum; this is why low-onset upper-layer tokens
(\texttt{The}, \texttt{ H}, \texttt{anoi}) show a displayed column total below
$\Delta p$---their final-segment mass sits in $L_{16}$ (e.g.
$\IG_{L_{16}}=0.59$ for \texttt{anoi}). Each full row sum matches $\Delta p$ to
floating point \emph{at this step count}, since the default estimator carries no
Riemann error.}
\label{tab:beijing}
\end{table}

% =========================================================================
\section{Evaluation Plan}\label{sec:eval}
% =========================================================================
The internal completeness checks establish that IG-Lens computes what it
claims. They do \emph{not} establish that the derived onset layer is
\emph{meaningful}, because onset has no intrinsic ground truth. We therefore
propose to validate onset by an \emph{independent consequence} rather than by
self-consistency:
\begin{enumerate}[leftmargin=1.4em,itemsep=2pt]
\item \textbf{Difficulty correlation.} Test whether $L^\star$ correlates with
token difficulty (e.g. the logit gap), and whether the IG-Lens onset correlates
more strongly than the lens-crossing onset.
\item \textbf{Causal ablation (primary).} If IG-Lens places a token's onset at
layer $L$, then ablating or patching $L$ should depress the target probability
more than ablating a low-onset layer \cite{meng2022rome,wang2023ioi}. This has
an objective answer and turns ``is the onset reasonable'' into ``does the onset
predict the ablation-sensitive layer.''
\end{enumerate}
A negative result on the ablation test would falsify the claim cheaply, which we
regard as a feature of the design.

% =========================================================================
\section{Discussion}
% =========================================================================
\paragraph{What IG-Lens is, and is not.} The contribution is not the telescoping
trick in isolation---segment IG along a path is an elementary consequence of the
gradient theorem. The contribution is what the trick \emph{buys}: an additive
decomposition of \emph{probability}, not logits, with the softmax and LayerNorm
integrated through rather than linearized or frozen. Every prior additive
method we are aware of pays for additivity by leaving probability space
\cite{elhage2021mathematical,dar2023prisms} or by abandoning the telescoping sum
\cite{dhamdhere2019conductance}. IG-Lens keeps both, at the cost described next.

\paragraph{The price of exactness.} The head-once readout discards the effect of
$\mathbf{h}_L$ routed through the attention and MLP blocks above it. A reader who
wants the \emph{total} causal effect of a layer should not use IG-Lens; DLA-style
component decomposition is the right tool there. What IG-Lens measures is a
conditional marginal: the additional probability that norm$+$head can read out of
$\mathbf{h}_L$ given $\mathbf{h}_{L-1}$. We argue this is the appropriate
quantity for the onset question---``by which layer is the prediction readable''
is a statement about readout, not about downstream routing---but we do not claim
it is the only useful quantity, and a total-effect variant that trades away exact
completeness is a sensible companion.

\paragraph{Non-monotonicity is a feature.} Because the lens is non-monotone, a
token's probability can be built up and partially undone across layers. The
signed IG-Lens decomposition exposes this directly: a negative $\IG_L$ marks a
layer whose readout \emph{lowers} the target probability, information that a
level-reading lens hides and that a $|\cdot|$-based onset rule conflates with
positive contribution. This is precisely why we move to a signed onset
definition for evaluation, and it suggests a concrete empirical claim: that
logit-space and lens-crossing onsets misplace the onset of tokens whose
trajectory is non-monotone (typically saturated or hard tokens), while the
probability-exact decomposition does not.

\paragraph{Relation to the iterative-inference view.} IG-Lens can be read as a
quantitative refinement of the iterative-inference picture
\cite{nostalgebraist2020logitlens,elhage2021mathematical}: instead of asking
\emph{what} the running prediction is at each layer, it asks \emph{how much} of
the final probability each layer is responsible for, and guarantees those
responsibilities sum to the whole. In that sense it is to the logit lens what an
accounting identity is to a snapshot.

% =========================================================================
\section{Conclusion}
% =========================================================================
We introduced IG-Lens, a telescoping segment-IG attribution that gives an exact
additive decomposition of a predicted token's probability across transformer
layers, with the softmax inside the integration path. Its default estimator
credits each integration step its observed change in target probability rather
than its raw gradient---a prediction-aware reweighting in the spirit of IDGI---so
that, for a scalar probability readout, completeness holds to floating point at
any step count and steps with sensitivity but no output response drop out. The
method trains nothing and reduces to a single batched forward pass with no
autograd call, whose correctness follows from the same head-once design that
makes the sum exact. We positioned it as the probability-space, additive
counterpart to logit-space onset methods, and we were explicit about what it
gives up---the effect of a layer through upper blocks---and about what remains to
be shown. The central open question is empirical and has objective ground truth:
does the IG-Lens onset predict the layer at which causal ablation most depresses
the prediction, and does it do so more accurately than logit-space onsets on
non-monotone tokens? We have designed that test to be cheap to run and capable of
falsifying the claim, and we regard answering it as the natural next step.

% --- bibliography ---------------------------------------------------------
\bibliographystyle{plain}
\bibliography{ig_lens}

\end{document}